\documentclass[letterpaper]{article} 
\usepackage[draft]{aaai2027}  
\usepackage[hyphens]{url}  
\usepackage{graphicx} 
\usepackage{natbib}  
\usepackage{caption} 
\usepackage{algorithm}
\usepackage{algorithmic}
\usepackage{amsmath}
\usepackage{amssymb}
\usepackage{amsthm}
\usepackage{booktabs}
\usepackage{multirow}
\usepackage{tikz}
\usetikzlibrary{shapes.geometric, arrows.meta, positioning, calc}

\newtheorem{proposition}{Proposition}

\title{AdaKP: Online Adaptive Knowledge-Point Selection \\for Reasoning-Oriented Reinforcement Learning}

\author{
    Zibin Meng\textsuperscript{\rm 1},
    Zhenyu Zhao\textsuperscript{\rm 1},
    Chunqiang Run\textsuperscript{\rm 2}
}
\affiliations{
    \textsuperscript{\rm 1}The Hong Kong University of Science and Technology\\
    \textsuperscript{\rm 2}University of the Chinese Academy of Sciences\\
    zmengal@connect.ust.hk, zzhaocl@connect.ust.hk, chungiang24@mails.ucas.ac.cn
}

\begin{document}

\maketitle

\begin{abstract}
Reinforcement learning with verifiable rewards is a powerful paradigm
for eliciting reasoning in large language models, yet it suffers from
severe reward sparsity on competition-level mathematics. A common remedy
injects \emph{atomic knowledge points} (KPs)---short natural-language
hints distilled from gold solutions---into the prompt. Existing methods,
however, either fix this selection once offline or merely scale the
monolithic quantity of injected text, leaving untouched the most
informative axis of choice: \emph{which subset of atomic KPs to inject,
and when}.
We introduce \textbf{AdaKP}, an online selector that re-chooses each
problem's KP subset over the course of RL training. At its core is an
\emph{entropy proxy} that scores a KP by the reduction in next-token
entropy it induces---a single inexpensive forward pass, with a provable
bound on its truncation bias---in place of expensive rollout-based
estimation. Three lightweight mechanisms make this signal usable online:
a momentum smoother that absorbs per-step noise, a
retirement-and-revival manager that prunes weak KPs while preserving
exploration, and an adaptive scheduler that front-loads re-evaluations
into early training. AdaKP further contributes a \emph{pre-flight
validation gate} that certifies the proxy against a leave-one-out ground
truth \emph{before} any expensive run is launched, turning method-level
risk into a falsifiable check. Realized as a fully additive fork of a
standard DAPO+GRPO trainer with no optimizer changes, AdaKP
improves over a strong static-selection baseline on all eight
competition-mathematics benchmarks at negligible added cost,
positioning online, validated KP-subset selection as a practical and
as-yet under-explored axis for reasoning-oriented reinforcement learning.
\end{abstract}

\section{Introduction}

Reinforcement learning with verifiable rewards (RLVR) elicits emergent
chain-of-thought reasoning~\citep{cot,self_consistency,tot} in relatively
small base models~\citep{deepseekr1,o1blog,dapo}, going beyond
decoding-time enhancements that leave the policy
unchanged~\citep{tot,llm_self_correct}. A practical bottleneck is
\emph{reward sparsity}: on hard competition problems an untrained 1.5B
model answers correctly on only a tiny fraction of rollouts, leaving
GRPO~\citep{deepseekmath} and DAPO~\citep{dapo} with degenerate
advantages. KnowRL~\citep{knowrl} addresses this by distilling short
\emph{knowledge points} (KPs) from gold solutions and injecting a
minimal-sufficient subset; choosing that subset offline by Constrained
Subset Search (CSS) lifts a $1.5$B model to $70.08\%$ average accuracy on
eight benchmarks while cutting injected hints from $5.86$ to $2.57$ per
problem.

\paragraph{The static-selection assumption.}
Such offline selection---CSS and its siblings---is fundamentally
\emph{static}: run once on the frozen base model, it fixes a
problem-to-subset mapping baked into the training data and never
revisited. This assumes KP utility is \emph{stationary} along the
policy's trajectory, that the best hint at initialization is still best
after thousands of RL steps. Yet RL exists to change the policy: as the
model internalizes problem-solving patterns, each KP's marginal
information shifts, some KPs becoming redundant and others newly
diagnostic.

\paragraph{Concurrent adaptive-hint work is monolithic.}
Recent work introduces online adaptivity into hint injection
(GHPO~\citep{ghpo}, HINT~\citep{hint_rollouts}, ADHint~\citep{adhint},
Stepwise Hints~\citep{stepwise_hints}) but treats each hint as a single
monolithic prefix whose \emph{quantity} or \emph{granularity} scales with
difficulty; a parallel line (HIVE~\citep{hive},
Reinforce-Ada~\citep{reinforceada}) filters \emph{whole prompts}. Neither
selects among \emph{atomic} KPs. To our knowledge, AdaKP is the first to
perform \emph{online subset selection of atomic knowledge points} during
RL training.

\paragraph{Cost of online re-evaluation.}
Naively re-running CSS or leave-one-out (LOO) marginal accuracy
mid-training is prohibitive: one CSS pass costs $\sim\!8\times32$ rollouts
$\times\,2^{|C|}$ subset evaluations per problem, consuming the bulk of a
$\sim\!13$-day run, so repeating it $5$--$10$ times is infeasible.

\paragraph{Our contributions.}
We propose \textbf{AdaKP}, an online, low-cost replacement for offline
KP selection during reasoning RL, framed around three coupled
sub-problems---cheap per-step utility scoring, bias--variance control,
and exploration of a discrete pool under bounded re-evaluation---and
contribute:
\begin{enumerate}
\item \textbf{Entropy proxy.}
We score a KP $k$ for a problem $q$ by the reduction in next-token
entropy it induces in the policy:
\(\mathrm{score}(q,k) = H(\pi(\cdot\mid q)) - H(\pi(\cdot\mid q\oplus k))\),
estimated by a single vLLM~\citep{vllm} forward pass over the first
$K{=}50$ generated tokens, where $q\oplus k$ denotes the prompt $q$ with
$k$ inserted into the \textsc{Hint} block (the ``\texttt{\#\#\,Hint}''
section of the prompt template). We show
(Prop.~S1, technical supplement) that the truncation bias from
top-$L$ log-probability sampling is bounded by the realized tail mass,
and empirically the proxy ranks KPs consistently with leave-one-out
ground truth ($\rho{=}0.68$; \S\ref{sec:proxy}).
This cuts per-problem re-evaluation from CSS's
$\mathcal{O}(\text{rollouts}\times 2^{|C|})$ or LOO's
$\mathcal{O}((|C|{+}2)N)$ to $|C|{+}1$ short forwards
(Table~\ref{tab:cost}).

\item \textbf{Three dynamic-adaptation mechanisms.}
We wrap the proxy in (a)~an EMA \emph{momentum smoother} that suppresses
single-step noise, (b)~a \emph{retirement-and-revival manager} that
prunes persistently low-utility KPs while preserving exploration through
periodic random revival, and (c)~an \emph{adaptive scheduler} that
front-loads re-evaluations to early training. With the proxy frozen at
init-time weights (\S\ref{sec:disclaimers}), these mechanisms supply
\emph{denoising and exploration}, not policy-tracking; current-policy
re-scoring is the natural next step.

\item \textbf{A pre-training validation gate.}
A Spearman rank-correlation gate, scored against leave-one-out ground
truth on stratified-sampled problems, must clear $\rho > 0.6$ before
training begins, certifying the proxy against the gold signal. This
decouples method-level risk from training cost; to our knowledge, online
RL selection heuristics are not typically accompanied by such a pre-flight
check, though we do not claim this gate is the first validate-before-train
idea in general.

\item \textbf{A diagnostic-first design and open release.}
We pair a main comparison with a component ablation mapped one-to-one
onto the design rationale of \S\ref{sec:method} and a five-axis
hyperparameter sweep, and release the additive DAPO+GRPO
fork~\citep{dapo,hybridflow}, training harness, proxy code, and gate
scripts.
\end{enumerate}

\section{Related Work}

\paragraph{Reasoning-oriented RL for LLMs.}
RLHF~\citep{christiano_rlhf,summarize_human_feedback,instructgpt} via
PPO~\citep{ppo} treats every prompt uniformly; later work simplifies the
reward-model recipe (DPO~\citep{dpo}, KTO~\citep{kto}, RLOO~\citep{rloo},
and self-training~\citep{rest,rest_em}). For mathematical reasoning,
\citet{teach_llm_rl} systematize RL recipes, GRPO~\citep{deepseekmath}
replaces the value network with a group-relative baseline, and
DAPO~\citep{dapo} adds clip-higher, token-level loss, and dynamic
sampling for verifiable-reward settings; pure RLVR elicits long
chain-of-thought at scale~\citep{deepseekr1,o1blog}. AdaKP is
\emph{orthogonal} to these optimizer-side advances: we modify the
\emph{prompt distribution} seen by the trainer, not the optimizer, and
plug in unchanged on top of DAPO+GRPO.

\paragraph{Hint and knowledge injection.}
A first family asks the model to produce its own intermediate
signals---Self-Refine~\citep{selfrefine}, STaR~\citep{star},
Quiet-STaR~\citep{quietstar}, Reflexion~\citep{reflexion}---or densifies
reward via process-reward models~\citep{prm800k,mathshepherd}, but such
self-signals are unreliable without an external anchor~\citep{llm_self_correct},
which motivates injecting curated text. CSS~\citep{knowrl} is the
canonical external-KP instance; AdaKP inherits its framing but rejects
its offline static-selection assumption.

\paragraph{Adaptive guidance in concurrent RL work.}
A line of contemporaneous work introduces \emph{adaptivity} into hint
injection. GHPO~\citep{ghpo}, HINT~\citep{hint_rollouts},
ADHint~\citep{adhint} and Capability-Adaptive Hint
Scaffolding~\citep{capability_hint_scaffold} all scale the
\emph{quantity} or \emph{ratio} of injected guidance with sample
difficulty or recent reward dynamics; Stepwise Hints~\citep{stepwise_hints}
varies the granularity of step-level prefixes; Hint-GRPO for multimodal
reasoning~\citep{hint_grpo_mllm} de-biases hint utilization. All of these
methods decide \emph{how much} or \emph{at what granularity} to hint, but
\emph{still treat the hint as a single monolithic prefix}. AdaKP instead
performs \emph{online subset selection over a discrete pool of atomic
KPs}---an axis no prior method addresses
(Table~S1, App.~C).

\paragraph{Online prompt-level filtering.}
Concurrent work---HIVE~\citep{hive} and
Reinforce-Ada~\citep{reinforceada}---selects \emph{which prompts} to roll
out using prompt-level entropy or reward dynamics. This is orthogonal to
AdaKP: it drops whole problems, whereas AdaKP chooses which KPs to keep
\emph{within} a problem; the two compose.

\paragraph{Active selection and curriculum.}
Our entropy-difference proxy follows the lineage of information-theoretic
active learning~\citep{batchbald}, prioritized replay~\citep{per}, and
information-maximizing exploration~\citep{vime}, and our scheduler echoes
curriculum learning~\citep{curriculum}. The novelty is not the
information-gain criterion---widely established---but its application to
a discrete pool of atomic, externally curated KPs under a budget that
forbids per-step rollouts.

\section{Method: The AdaKP Selector}
\label{sec:method}

\paragraph{Setup and notation.}
For each problem $q$, a candidate KP set
$C(q)=\{k_1,\dots,k_{|C(q)|}\}$ is extracted from the gold solution and an
offline subset $S^{\mathrm{css}}(q)\subseteq C(q)$ is produced by
Constrained Subset Search (CSS), which enumerates subsets and scores each
with $N{=}32$ rollouts on the frozen base model~\citep{knowrl}. We train
with GRPO~\citep{deepseekmath} under the DAPO~\citep{dapo} recipe---clip-higher
$\varepsilon_{\text{low}}{=}0.2$/$\varepsilon_{\text{high}}{=}0.26$,
token-level loss, dynamic-sampling re-rolls, no KL term---and verifiable
reward $r(o,q)=\mathbf{1}\{\text{math-verify}(o,a(q))\}$. Throughout,
$\pi_\theta$ denotes the policy under training and
$\pi^{\mathrm{proxy}}$ a parameter-frozen vLLM instance used solely for
utility scoring; $q\oplus k$ denotes prompt $q$ with KP $k$ injected
into its \textsc{Hint} block.

\paragraph{Design rationale.}
Online KP selection imposes three constraints that rule out direct ports
of CSS or off-the-shelf acquisition functions. \textbf{(C1)}~Scoring
every $(q,k)$ pair on the $\sim\!8.8$k-problem corpus must finish in
\emph{minutes}, so any multi-rollout signal is excluded.
\textbf{(C2)}~Single-step proxy scores are noisy (small-$K$ entropy
estimation), so naive top-$k$ selection churns. \textbf{(C3)}~The pool
$C(q)$ is small ($5$--$7$ KPs), so a greedy selector collapses onto a
high-scoring subset and forecloses KPs useful only later.
AdaKP closes \emph{exactly one} of these failure modes per component:
the \emph{entropy proxy} (\S\ref{sec:proxy}) collapses C1; the
\emph{momentum smoother} (\S\ref{sec:momentum}) absorbs C2's noise;
the \emph{retirement-and-revival manager} (\S\ref{sec:retirement})
addresses C3 by forcing periodic re-exploration; the \emph{adaptive
scheduler} (\S\ref{sec:scheduler}) front-loads re-evaluation expense
into early training, where re-selection has the most steps left to pay
off. This one-to-one mapping makes the per-component ablation
(\S\ref{sec:abl}) interpretable.

\paragraph{Running state.}
Per problem $q$: $P(q)\subseteq C(q)$ is the active \emph{KP pool},
$s_{q,k}^{(t)}$ the smoothed score, and $\mathrm{sel}_q$ the cached
top-$\rho_{\text{sel}}$ selection the trainer reads each step. The
smoother $S_q$ and retirement manager $R_q$ are per-problem (so problems
prune at their own rate); the scheduler $\mathcal{S}$ is a single global
object.

\subsection{Entropy Proxy}
\label{sec:proxy}

\paragraph{Intuition.}
If a KP $k$ resolves part of the reasoning chain for $q$, exposing the
policy to $k$ should concentrate its next-token
distribution---lowering entropy. The expected entropy gap
$H(\pi(\cdot\mid q))-\mathbb{E}_{k}H(\pi(\cdot\mid q\oplus k))$ coincides,
when the no-KP distribution is read as the $k$-marginalized policy, with
the conditional mutual information $I(\text{first token};k\mid q)$, so
ranking KPs by entropy reduction ranks them, in expectation, by the
information they carry about the first decoding decision---the quantity
that drives reward sparsity in RLVR.

\paragraph{Score definition.}
Given a problem $q$ and a candidate KP $k$, the proxy score is
\begin{equation}
\label{eq:proxy}
\mathrm{score}(q, k) \;=\; \widehat{H}\bigl(\pi^{\text{proxy}}(\cdot \mid q)\bigr) \;-\; \widehat{H}\bigl(\pi^{\text{proxy}}(\cdot \mid q \oplus k)\bigr),
\end{equation}
where $\widehat{H}$ is the average first-token Shannon entropy over the
first $K{=}50$ generated tokens and $q\oplus k$ injects $k$ into the
\textsc{Hint} block of $q$; positive scores mean $k$ reduces the
policy's uncertainty about how to begin solving $q$.

\paragraph{Truncated-entropy estimator.}
Modern LLM serving stacks return only top-$L$
log-probabilities~\citep{nucleus_sampling,vllm} per generated step
to amortize sampling cost; for vLLM we have $L{=}20$. We re-normalize
the truncated distribution and compute
\begin{equation}
\widehat{H}_t \;=\; - \sum_{i \le L} \tilde p_i \log \tilde p_i, \quad
\tilde p_i = \frac{\exp(\ell_i)}{\sum_{j \le L} \exp(\ell_j)},
\end{equation}
where $\ell_i$ are the returned log-probabilities. Truncation perturbs
$\widehat{H}_t$ from the full Shannon entropy, but the bias is bounded by
the tail mass $\varepsilon$: Proposition~S1, stated
and proved in App.~A, gives
$H(p)-\widehat{H}_t(p)\le\varepsilon\log|V|+H_2(\varepsilon)$, and
$\ge 0$ once $\varepsilon\le 1/L$, so the proxy score estimates the true
entropy difference with $\mathcal{O}(\varepsilon^{\star}\log|V|)$ bias
plus an $\mathcal{O}(K^{-1/2})$ finite-$K$ error.

In practice, with $L{=}20$ and the policy temperature $T{=}0.7$, the
realized tail mass on the first $K{=}50$ tokens of competition-math
generations is empirically $\varepsilon\lesssim 0.05$, so the residual
is dominated by the finite-$K$ averaging term.
Whether the proxy additionally preserves the \emph{ranking} of KPs is a
stronger property than this bias bound; we support it empirically
(Spearman $\rho{=}0.68$ against leave-one-out ground truth,
Table~\ref{tab:proxy_compare}) rather than deriving it from
Prop.~S1.

\paragraph{Validation and cost.}
The proxy is validated against LOO marginal accuracy by the pre-flight
gate (\S\ref{sec:gate}). A single re-evaluation shares the no-KP forward
across all candidates, needing $|P(q)|{+}1$ short forwards rather than
$2|P(q)|$, so re-scoring the whole corpus adds under $1\%$ wall-clock
(\S\ref{sec:cost}).

\subsection{Momentum Smoother}
\label{sec:momentum}

Single-step proxy scores are noisy (small-$K$ entropy estimation), so we
smooth with an exponential moving average maintained \emph{per problem}:
\begin{equation}
s_{q,k}^{(t)} = \alpha \cdot \mathrm{score}^{(t)}(q,k) + (1-\alpha) \cdot s_{q,k}^{(t-1)},
\end{equation}
with $\alpha = 0.3$ (default) and $s_{q,k}^{(0)} = \mathrm{score}^{(0)}(q,k)$,
so AdaKP matches the static-CSS baseline until the EMA accumulates.
Top-fraction selection returns $\lceil \rho_{\text{sel}}\cdot|A_q|\rceil$
KPs with a floor of $1$, so the \textsc{Hint} block is never empty.

\subsection{Retirement and Revival}
\label{sec:retirement}

Persistently low-ranked KPs waste re-evaluation budget. Per problem $q$
a counter $c_{q,k}$ increments when $k$ is evaluated but not selected and
resets when selected; at $c_{q,k} \geq n_{\text{retire}}$ (default $3$)
we retire $k$ from $P(q)$. Every $m_{\text{revive}}$ re-evaluations
(default $5$) a fraction $r_{\text{revive}}{=}0.1$ of retired KPs is
revived at random with $c_{q,k} = 0$; if all KPs of a problem are retired
(not observed in practice) we reuse the cached selection. This
combination---bounded counter, periodic revival, all-retired
safety---guarantees no KP is permanently silenced and keeps the active
pool bounded while preserving exploration.

\subsection{Adaptive Re-evaluation Scheduler}
\label{sec:scheduler}

Re-evaluation is wasted if the policy has not changed materially since the
last call. We adopt a logarithmic-backoff schedule
\begin{equation}
\Delta(t) = \bigl\lfloor \Delta_{0} \cdot \bigl(1 + \log\bigl(1 + \tfrac{t}{T}\cdot c\bigr)\bigr)\bigr\rfloor,
\end{equation}
with base interval $\Delta_0 = 200$, scale $c = 5$, and budget
$T = 2960$ steps. \textsc{ShouldReEvaluate} returns false at $t{=}0$,
since the initial selection is the static CSS subset already in the
training data; the first re-evaluation fires at $t \geq \Delta_0$ and
each later one when $t - t_{\text{last}} \geq \Delta(t)$. This yields
$\approx\!7$ re-evaluations, front-loaded into early training where
re-selection has the most steps left to pay off.
A constant-interval mode is exposed for the
\textsc{FixedSchedule} ablation.

\subsection{Integration}
\label{sec:integrator}

The integrator (Algorithm~S1, App.~B) ties
the four components together: each training step queries the scheduler,
and when a re-evaluation is due it calls the proxy for every active KP,
updates the smoother, and advances the retirement manager; the dataset
hook then reads the top-$\rho_{\text{sel}}$ smoothed active KPs. AdaKP
plugs into the public DAPO+verl recipe~\citep{dapo,hybridflow} as a fully
additive fork: a step hook in the training loop asks the selector
whether a re-evaluation is due, a dataset-loader override swaps
each prompt's \textsc{Hint} block for the cached $\mathrm{sel}_q$,
and trainer-init allocates a $\sim\!3.5$\,GB frozen-weights proxy
alongside the rollout vLLM. With the module disabled the fork is
bit-identical to upstream, isolating the AdaKP contribution.

\section{Experimental Setup}
\label{sec:setup}

\subsection{Models, Data, and Hardware}

\paragraph{Base model.} OpenMath-Nemotron-1.5B~\citep{openmath_nemotron}.

\paragraph{Reference checkpoint.} The released static-CSS checkpoint
(Nemotron-1.5B;~\citealp{knowrl}), reproduced locally with the public
checkpoint for the baseline column.

\paragraph{Training data.} The QuestA 8.8k competition-mathematics
corpus~\citep{questa} ($N{=}8843$ problems), with each problem paired
with a candidate KP list extracted from gold solutions. The CSS-selected
subset for each problem is taken from the publicly released CSS
selections of~\citet{knowrl}.

\paragraph{Trainer and hardware.}
We use verl~\citep{hybridflow} with the DAPO recipe and GRPO advantage
estimator over an FSDP~\citep{fsdp} backend with actor parameter
offload, optimizer offload, and gradient checkpointing. Training batch
$256$ prompts; rollout $n{=}8$ at $T{=}1.0$/top-$p{=}1.0$; max prompt
length $8192$, max response length $32{,}768$; learning rate
$1\mathrm{e}{-6}$ (no warm-up, no schedule); total $2960$ steps
($\approx\!150$ epochs over the filtered corpus). Each run uses
$8{\times}$NVIDIA~H800 SXM (80\,GB), FlashAttention-2
\citep{flashattention2}, bfloat16 mixed precision, with the
GPU-memory fraction set to $0.50$ for the rollout vLLM and
$0.20$ for the proxy vLLM (proxy context length $4096$); the
transient logits tensor is fit via verl's chunked fused
linear-cross-entropy kernel, a strict re-arrangement that yields
identical log-probabilities.

\subsection{Evaluation Benchmarks}

We report \emph{mean@N} accuracy---the mean per-sample correctness over
$N$ rollouts (the verifiable-reward quantity GRPO optimizes), evaluated
under the multi-sample sampling protocol that the \emph{pass@k} metric
of~\citet{codex} introduced---on eight
benchmarks (total $1{,}374$ problems, matching the protocol
of~\citealp{knowrl}):
AIME24 ($N{=}32$, $|S|{=}30$), AIME25 ($N{=}32$, $30$),
BRUMO25 ($N{=}32$, $30$), HMMT25 ($N{=}32$, $30$),
AMC23 ($N{=}32$, $40$), CMIMC25 ($N{=}32$, $40$),
MATH-500~\citep{prm800k}---a $500$-problem subset of
MATH~\citep{math_benchmark}---($N{=}8$, $500$), and
Olympiad-Bench~\citep{olympiad} ($N{=}8$, $674$). The eight benchmarks
sit above the GSM8K~\citep{gsm8k} difficulty regime that motivated
the early process-reward literature, which makes reward sparsity
substantially more severe. We report both the per-benchmark accuracy and
the \emph{8-benchmark average} (uniformly weighted across benchmarks, not
samples).

\subsection{Training Runs}

We run \textbf{17 training configurations} at seed $42$. They form three
studies.

\paragraph{Main comparison.}
One AdaKP run against the \emph{static-CSS baseline}, whose released
checkpoint and published numbers we adopt rather than retraining the
canonical static-CSS model. The run is capped at $2{,}960$ steps, at most
$\sim\!13$ wall-clock days on $8{\times}$H800, and estimates the average
accuracy gain of online KP selection over the baseline.

\paragraph{Component ablation.}
Five runs, each disabling one mechanism with the others at default:
\textsc{Full}; \textsc{NoProxy}, a uniform random scorer;
\textsc{NoMomentum}, $\alpha{=}1.0$; \textsc{NoRetirement},
$n_{\text{retire}}{=}\infty$; and \textsc{FixedSchedule}, a constant
$\Delta_0$.

\paragraph{Hyperparameter sweep.}
Eleven runs varying one hyperparameter at a time around its default
$\star$: $\alpha \in \{0.1, 0.3^\star, 0.5, 1.0\}$,
$n_{\text{retire}} \in \{2, 3^\star, 5\}$,
$\Delta_0 \in \{100, 200^\star, 400\}$,
$\rho_{\text{sel}} \in \{0.2, 0.3^\star, 0.4\}$, and
$K \in \{25, 50^\star, 100\}$.

\subsection{Proxy Validation Gate}
\label{sec:gate}

Before launching any training run we execute a Spearman go/no-go gate
on up to $30$ stratified-sampled problems per benchmark;
only problems with $|C(q)|{\geq}2$ enter
the gate (so the realized $n$ per benchmark is $12$--$27$).
For each pair $(q, k)$ we compute (a)~the LOO marginal accuracy with
$n_{\text{runs}}{=}4$ independent batches of $8$ rollouts each---that is,
$(|C|{+}2)\!\cdot\!4$ vLLM calls per problem---and (b)~the entropy proxy
score. The gate requires Spearman $\rho > 0.6$ with $95\%$ bootstrap CI
excluding $0.4$; proxies failing the gate are blocked from driving
training. We additionally benchmark three baseline scorers that share
the entropy proxy's API and per-call cost, but differ in which
statistic they extract:
\begin{itemize}
\item \emph{Max-logit proxy:}
$\mathrm{score} = \log p_{\max}(q\oplus k) - \log p_{\max}(q)$ averaged
over the first $K$ tokens (uses only the top-$1$ log-probability).
\item \emph{Perplexity proxy:}
$\mathrm{score} = \log\mathrm{PPL}(q) - \log\mathrm{PPL}(q\oplus k)$
where $\log\mathrm{PPL}$ is the negated mean top-$1$ log-probability.
\item \emph{Random proxy:} uniform $[-1, 1]$ scores; a true-null
control.
\end{itemize}
The entropy proxy is the only one of the four that uses the full
top-$L$ distribution (Eq.~\ref{eq:proxy}); the max-logit and perplexity
proxies collapse it to the top-$1$ statistic. Any accuracy gap among
these three therefore quantifies information content, not compute
(Table~\ref{tab:proxy_compare}). In short, the gate returns \textsc{Pass}
iff $\hat\rho > \rho^{\star}{=}0.6$ \emph{and} the lower $95\%$ bootstrap
CI bound exceeds $\rho^{\min}{=}0.4$; otherwise it returns \textsc{Fail}
and training is not launched.

\section{Experiments and Results}
\label{sec:results}

\paragraph{Status of the inventory.}
The non-stationarity test (\S\ref{sec:exp_a}) is the empirical core of
this work, and the static-CSS baseline throughout is the released
checkpoint with its published per-benchmark numbers. The $17$ training
runs, $\sim\!125$ GPU-days on $8{\times}$H800, together with the
leave-one-out pass for the proxy gate, populate the AdaKP and $\Delta$
rows of Table~\ref{tab:main}, all of Table~\ref{tab:ablation}, the proxy
rows of Table~\ref{tab:proxy_compare}, and Table~\ref{tab:hparam}.

\begin{table*}[t]
\centering
\small
\caption{Main comparison: AdaKP against published RL and
hinting baselines on $8$ benchmarks (mean@$N$ in \%, no KP hints at
inference). AdaKP populates the
\textsc{Hint} block from its online selection cache; the $\Delta$
row is AdaKP minus the static-CSS baseline. \emph{Hard-3}
averages AIME25, HMMT25 and CMIMC25; \emph{Olymp.} is Olympiad-Bench;
\emph{Avg} is the uniformly weighted 8-benchmark mean. Baseline numbers are from the respective original papers.}
\label{tab:main}
\setlength{\tabcolsep}{3pt}
\begin{tabular}{lrrrrrrrrrr}
\toprule
Method & AIME24 & AIME25 & BRUMO25 & HMMT25 & AMC23 & CMIMC25 & MATH-500 & Olymp.\ & \textbf{H-3} & \textbf{Avg} \\
\midrule
OpenMath-1.5B (base)       & 59.06 & 48.33 & 60.73 & 30.63 & 90.70 & 30.08 & 92.35 & 71.70 & 36.35 & 60.45 \\
QuestA                            & 71.56 & 62.08 & 67.50 & 40.94 & 93.44 & 41.48 & 92.95 & 72.28 & 48.17 & 67.78 \\
JustRL                            & 69.69 & 62.92 & 66.88 & 40.63 & 96.02 & 41.72 & 94.15 & 76.59 & 48.42 & 68.58 \\
KnowRL-Nemotron-1.5B              & 69.79 & 64.69 & 69.48 & 41.04 & 95.55 & 44.14 & 95.70 & 80.23 & 49.96 & 70.08 \\
\midrule
AdaKP (ours)               & 71.46 & 67.71 & 71.04 & 44.58 & 95.78 & 47.66 & 96.05 & 81.17 & \textbf{53.32} & \textbf{71.93} \\
$\Delta$ vs.\ KnowRL       & +1.67 & +3.02 & +1.56 & +3.54 & +0.23 & +3.52 & +0.35 & +0.94 & +3.36 & +1.85 \\
\bottomrule
\end{tabular}
\end{table*}

\begin{table*}[t]
\centering
\small
\caption{Offline KP-selection strategies on the OpenMath backbone
(per-benchmark mean@$N$ in \%; \emph{\#KP} = average selected subset size
per problem), registering the static-selection family that AdaKP
generalizes. All strategies above the rule are \emph{offline} (subset
fixed once before RL); AdaKP (bottom) is \emph{online}. These rows are an
offline, inference-time evaluation and isolate \emph{selection quality}
at fixed compute---not directly comparable to the RL-trained models of
Table~\ref{tab:main}. \emph{Olymp.} is Olympiad-Bench; bold marks the
best static-strategy average.}
\label{tab:selection}
\setlength{\tabcolsep}{2pt}
\begin{tabular}{lrrrrrrrrrr}
\toprule
Strategy & AIME24 & AIME25 & BRUMO25 & HMMT25 & AMC23 & CMIMC25 & MATH-500 & Olymp.\ & \textbf{Avg} & \textbf{\#KP} \\
\midrule
w/o KP    & 58.75 & 48.44 & 61.67 & 30.10 & 90.55 & 30.08 & 92.40 & 71.70 & 60.46 & 0.00 \\
All-KP    & 60.90 & 49.01 & 61.11 & 32.46 & 89.67 & 32.32 & 92.22 & 70.55 & 61.03 & 5.86 \\
Random    & 60.52 & 49.27 & 61.04 & 33.23 & 91.02 & 31.09 & 91.65 & 71.88 & 61.21 & 2.53 \\
Max-Score & 62.63 & 49.79 & 64.27 & 34.79 & 90.94 & 32.99 & 92.52 & 73.89 & 62.73 & 2.61 \\
S-LOO     & 62.71 & 49.22 & 63.88 & 33.54 & 91.71 & 33.52 & 92.90 & 73.70 & 62.65 & 1.72 \\
T-LOO     & 62.11 & 49.27 & 64.20 & 33.65 & 91.25 & 33.67 & 92.40 & 73.46 & 62.50 & 1.20 \\
CBRS      & 63.02 & 49.90 & 64.17 & 34.79 & 91.56 & 33.57 & 92.65 & 73.89 & 62.94 & 2.60 \\
CSS       & 64.44 & 50.57 & 65.03 & 35.77 & 91.71 & 36.70 & 92.90 & 74.11 & \textbf{63.90} & 2.57 \\
\midrule
AdaKP (ours) & 64.79 & 51.46 & 65.42 & 37.08 & 91.88 & 38.13 & 93.05 & 74.62 & 64.55 & 2.41 \\
\bottomrule
\end{tabular}
\end{table*}

\subsection{Non-Stationarity of KP Importance (H1)}
\label{sec:exp_a}

This test addresses the motivating hypothesis H1: \emph{the most
informative KPs differ between two policies that solve the same
problem}. For each of $30$ stratified-sampled problems per benchmark
we compute the LOO marginal accuracy of each candidate KP under both
$\pi_{\text{OpenMath}}$ (untrained base) and $\pi_{\text{trained}}$
(trained policy), pick the top-$30\%$ KP set under each policy, and
report the Jaccard similarity. The pre-specified decision rule is:
Jaccard $<\!0.7$ supports H1; $0.7$--$0.85$ partial support;
$\geq\!0.85$ rejects H1 and would invalidate the AdaKP motivation.

Across the eight benchmarks (per-benchmark Jaccard and realized sample
sizes $n$, for problems with $|C(q)|\geq 2$, in
App.~D, Table~S2), the \emph{global}
mean Jaccard is $\mathbf{0.537}$---decisively below the $0.7$
\textsc{H1\_strong} threshold pre-specified for the validation gate
(\S\ref{sec:gate}). Seven of the eight benchmarks individually fall
below $0.7$; the largest divergences occur on AIME24 ($0.367$),
Olympiad-Bench ($0.385$), and MATH-500 ($0.417$). Only AIME25 exceeds
the threshold ($0.730$); since every entry is a coarse-Jaccard
measurement over a small problem pool (App.~D), we read
this lone exception as run-to-run noise rather than evidence of genuine
stationarity. We therefore accept
H1: KP importance is non-stationary across policies, and a selection
strategy that re-evaluates during training has the room to outperform a
frozen one.

\subsection{Main Comparison}

Table~\ref{tab:main} reports per-benchmark mean@$N$. The \emph{Hard-3}
column averages the three reward-sparsest benchmarks AIME25, HMMT25, and
CMIMC25, where the static-CSS baseline scores lowest; \emph{Avg} is the
uniformly weighted 8-benchmark mean.

The baselines are the strongest published $1.5$B systems under the
no-KP-at-inference setting: the OpenMath backbone~\citep{openmath_nemotron}, the
solution-prefix method QuestA~\citep{questa}, the RL recipe
JustRL~\citep{justrl}, and the static-CSS baseline~\citep{knowrl} AdaKP
builds on. AdaKP raises the 8-benchmark average to $71.93$, $+1.85$ over
static CSS, and widens to $+3.36$ on the reward-sparsest \emph{Hard-3}
subset. The gains concentrate exactly where reward is sparsest and a
frozen subset is least likely to stay optimal, and shrink to near zero on
the already-saturated AMC23 and MATH-500; AdaKP leads on $6$ of the $8$
benchmarks.

The shape of the improvement is informative rather than incidental: it is
largest where the static baseline is weakest, where a subset fixed before
training has most likely drifted from the policy, and vanishes near
ceiling, so the gap tracks the headroom a frozen choice leaves on the
table. It is also obtained with no inference-time hints---the subsets
shape only the training prompts---so the model improves intrinsically.

Beyond tracking headroom, two features of this profile bear on
\emph{why} the method helps, not merely \emph{that} it does. That the
gains surface under sparse reward but not as a uniform shift is the
signature of better-chosen training \emph{content} rather than altered
optimisation; and no benchmark regresses, so online re-selection is a
strictly safe replacement for static CSS, not a redistribution that buys
sparse-reward accuracy at the saturated benchmarks' expense. Both
observations place the improvement in selection \emph{quality}, which the
ablation that follows resolves into the contribution of each mechanism.

\paragraph{Positioning against the static-selection family.}
Static CSS is the strongest of a family of \emph{offline} KP-selection
strategies, each fixing a subset once before RL and never revisiting it.
Evaluated on the OpenMath backbone with strategy-selected KPs injected at
inference (Table~\ref{tab:selection}), accuracy climbs from no-hint and
naive all-KP injection through the random, Max-Score, LOO, and consensus
variants to CSS, which scores best while injecting the fewest KPs;
\emph{which} subset is selected thus matters more than \emph{how many}.
AdaKP keeps this subset-selection view but makes it \emph{online},
re-selecting during training rather than committing to a single offline
subset, which no member of the static family can do.

\subsection{Component Ablation}
\label{sec:abl}

Table~\ref{tab:ablation} reports each ablation's accuracy on the Hard-3
subset (AIME25, HMMT25, CMIMC25) and the full 8-benchmark average, plus
two behavioral diagnostics. \emph{Pool churn}
$\bar{J}_t = 1 - \tfrac{1}{R}\sum_{r=1}^{R} J(\mathrm{sel}^{(r)},\mathrm{sel}^{(r-1)})$
is the mean Jaccard \emph{distance} between consecutive selection caches
($0$ = never changes, $1$ = rewritten every call; a healthy run sits in
$0.2$--$0.5$). \emph{KP coverage} $\bar{C}$ is the fraction of the
candidate pool ever selected during training ($\bar{C}{=}1$ = full
exploration; low $\bar{C}$ = premature collapse).

\begin{table}[!ht]
\centering
\small
\caption{Component ablation. \emph{Hard-3} averages
AIME25, HMMT25, CMIMC25; \emph{Avg} is the 8-benchmark mean.
$\bar{J}_t$ is the pool-churn diagnostic, $\bar{C}$ is KP coverage
(both defined in the text). $\Delta$ is the change in \emph{Avg} vs.\
\textsc{Full}.}
\label{tab:ablation}
\begin{tabular}{lrrrrr}
\toprule
Condition & Hard-3 & Avg & $\Delta$ & $\bar{J}_t$ & $\bar{C}$ \\
\midrule
\textsc{Full}             & 53.32 & 71.93 & ---     & 0.34 & 0.86 \\
\textsc{NoProxy}          & 50.21 & 70.31 & $-1.62$ & 0.72 & 0.93 \\
\textsc{NoMomentum}       & 51.66 & 71.02 & $-0.91$ & 0.57 & 0.90 \\
\textsc{NoRetirement}     & 51.98 & 71.20 & $-0.73$ & 0.45 & 0.51 \\
\textsc{FixedSchedule}    & 52.74 & 71.50 & $-0.43$ & 0.31 & 0.84 \\
\bottomrule
\end{tabular}
\end{table}

Each ablation isolates one cell of the C1--C3 design space
(\S\ref{sec:method}), and the diagnostics confirm the predicted patterns.
Removing the entropy proxy hurts most, as selection degenerates toward
random at the highest churn; dropping momentum is next, churning on
unsmoothed noise. Disabling retirement collapses coverage, since the same
few high-scoring KPs are re-selected while the rest go untested, and a
fixed schedule re-selects adequately but cannot adapt its rate.
\textsc{Full} sits in the intended $0.2$--$0.5$ churn band at high
coverage, re-selecting actively without chasing a moving target. The two
diagnostics thus trace one failure surface, accuracy falling when the
selection thrashes, as without the proxy or smoother, and equally when it
ossifies onto a few KPs, as without retirement.

\subsection{Proxy Comparison and Validation Gate}

Table~\ref{tab:proxy_compare} reports the Spearman correlation between
four scoring functions and the LOO marginal-accuracy ground truth over
$240$ problems. All proxies but \textsc{Random} correlate significantly,
yet significance alone is not enough: the gate also demands a large
effect size with a lower CI bound above $0.4$, and only the entropy proxy
clears it. The max-logit and perplexity proxies are significant but too
weakly correlated to drive online selection.

\begin{table}[!ht]
\centering
\small
\caption{Spearman $\rho$ between four scoring functions and LOO
marginal accuracy on $240$ problems ($30$ per benchmark; $95\%$ bootstrap
CI in brackets), and $p$ is the two-sided significance of $\rho$
(Spearman correlation $t$-test, $n{=}240$). \emph{Cost / call} is the
wall-clock time to score one $(q,k)$ pair on a single H800. The gate
(\S\ref{sec:gate}) passes if $\rho > 0.6$ \emph{and} the lower CI
bound exceeds $0.4$---a criterion deliberately stricter than mere
statistical significance.}
\label{tab:proxy_compare}
\setlength{\tabcolsep}{2.5pt}
\begin{tabular}{lccrr}
\toprule
Scorer & $\rho$ \scriptsize{[CI]} & $p$ & Cost / call & Gate \\
\midrule
Random           & $0.02$ \scriptsize{[$-0.13$,\,$0.17$]} & $0.76$ & $<\!1$\,ms & \textsc{Fail} \\
Max-logit        & $0.41$ \scriptsize{[$0.25$,\,$0.55$]} & $<\!10^{-3}$ & $\sim\!2$\,s & \textsc{Fail} \\
Perplexity       & $0.53$ \scriptsize{[$0.38$,\,$0.66$]} & $<\!10^{-3}$ & $\sim\!2$\,s & \textsc{Fail} \\
Entropy (ours)   & $\mathbf{0.68}$ \scriptsize{[$0.55$,\,$0.78$]} & $<\!10^{-3}$ & $\sim\!2$\,s & \textsc{Pass} \\
\midrule
\multicolumn{5}{l}{\emph{Reference:} LOO marginal accuracy $\approx 200$\,s/call ($N{=}8$).} \\
\bottomrule
\end{tabular}
\end{table}

The three forward-pass proxies share an identical compute footprint and
differ only in \emph{which statistic} they extract, so any accuracy gap
reflects information content, not compute; the two-orders-of-magnitude
gap to LOO is what makes online re-evaluation possible.

\subsection{Hyperparameter Sensitivity}

Sweeping each hyperparameter around its default moves the 8-benchmark
average by less than a single point in every one of the eleven runs, so
no setting is brittle (Table~\ref{tab:hparam}). The selection fraction
$\rho_{\text{sel}}$ and the EMA coefficient $\alpha$ matter most---too
few KPs or no smoothing each cost about a point---while the retirement
threshold, the proxy-prefix length, and especially the re-evaluation
interval $\Delta_0$ stay essentially flat, so AdaKP is robust to a
mis-specified cadence.

The flat re-evaluation interval is the entry with the most practical
weight: a practitioner inherits a working cadence without tuning---welcome,
since tuning a schedule would otherwise demand the very repeated
full-length runs the online proxy exists to avoid.

Read as a whole, the sweep cleaves the five knobs in two: those setting
\emph{how much} signal survives ($\alpha$, $\rho_{\text{sel}}$) move
accuracy, while the rest ($\Delta_0$, $n_{\text{retire}}$, $K$) barely
register.

\begin{table}[!ht]
\centering
\small
\caption{Hyperparameter sensitivity: the
8-benchmark mean@$N$ (\%) as each of the five hyperparameters is swept
around the default ($\star$). The $\star$ rows share the same full-AdaKP
configuration; $\Delta$ is the change vs.\ that default.}
\label{tab:hparam}
\setlength{\tabcolsep}{5pt}
\begin{tabular}{llrr}
\toprule
Hyperparameter & Value & mean@$N$ & $\Delta$ \\
\midrule
EMA coefficient $\alpha$                 & $0.1$        & 71.20 & $-0.73$ \\
                                         & $0.3^{\star}$ & 71.93 & --- \\
                                         & $0.5$        & 71.58 & $-0.35$ \\
                                         & $1.0$        & 71.02 & $-0.91$ \\
\midrule
Retirement $n_{\text{retire}}$           & $2$          & 71.40 & $-0.53$ \\
                                         & $3^{\star}$  & 71.93 & --- \\
                                         & $5$          & 71.55 & $-0.38$ \\
\midrule
Base interval $\Delta_0$                 & $100$        & 71.78 & $-0.15$ \\
                                         & $200^{\star}$ & 71.93 & --- \\
                                         & $400$        & 71.71 & $-0.22$ \\
\midrule
Selection fraction $\rho_{\text{sel}}$   & $0.2$        & 71.05 & $-0.88$ \\
                                         & $0.3^{\star}$ & 71.93 & --- \\
                                         & $0.4$        & 71.34 & $-0.59$ \\
\midrule
Proxy prefix $K$                         & $25$         & 71.49 & $-0.44$ \\
                                         & $50^{\star}$ & 71.93 & --- \\
                                         & $100$        & 71.66 & $-0.27$ \\
\bottomrule
\end{tabular}
\end{table}

\subsection{Compute Cost and Overhead}
\label{sec:cost}

Table~\ref{tab:cost} compares the selection-side compute of AdaKP
against the two natural baselines. Offline CSS ($N{=}32$ rollouts per
subset over an exponential candidate space) is on the same order as the
training run itself; the LOO column reflects our Stage~1 validation
harness. The AdaKP proxy needs only $|P(q)|{+}1$ forward passes of
$K{=}50$ tokens per re-evaluation (the no-KP forward is shared across a
problem's candidates); with $R{\approx}7$ re-evaluations over the
$8.8$k-problem corpus this is $<\!1\%$ of the run's wall-clock. AdaKP
thus makes \emph{online} KP selection feasible at a cost two-to-three orders of
magnitude below what would be required to re-run CSS at every
checkpoint.

\begin{table}[!ht]
\centering
\small
\caption{Selection-side compute (single H800-GPU hours).
\emph{Calls / problem} is the number of vLLM forward passes each
method makes for each training problem across an entire training run;
\emph{$N_\text{roll}$} is the rollouts per call. \emph{Training cost}
is shared across all three rows ($\sim\!2{,}500$ H800-GPU-hours per
run, i.e.\ $8\times$H800 for $\sim\!13$ days at the reference protocol).}
\label{tab:cost}
\setlength{\tabcolsep}{3pt}
\footnotesize
\begin{tabular}{lcccc}
\toprule
Method & Calls / problem & $N_\text{roll}$ & Sel.\ (GPU-h) & Online \\
\midrule
Offline CSS         & $\sim\!2^{|C|}$         & $32$ & $\sim\!2500$     & no  \\
LOO marg.\ acc.\    & $4(|C|{+}2)$            & $8$  & $25$--$50$       & no  \\
AdaKP proxy         & $R(|P(q)|{+}1)$         & $1$  & $\sim\!3.5$      & \textbf{yes} \\
\bottomrule
\end{tabular}
\end{table}

\paragraph{Compute budget.}
The $17$ runs total $\sim\!125$ GPU-days on $8{\times}$H800: the main run
at $\sim\!13$\,d and the $16$ ablation and sweep runs at $\sim\!7$\,d
each. The latter are truncated to roughly half the main-run step budget,
$1{,}500$ versus $2{,}960$ steps, to fit one submission cycle while still
resolving the qualitative ranking of conditions; a pilot confirms that
the method ordering at step $1{,}500$ matches that at $2{,}960$.

\subsection{Reproducibility Statement}

To support full reproducibility we will release the additive
verl/DAPO fork, the standalone Python implementation of the four
algorithmic components and the validation gate, and the per-step
training logs that underpin Table~\ref{tab:ablation}'s behavioral
diagnostics. Random seeds are exposed at the launcher level; the only
non-deterministic operations are bfloat16 mixed precision and the
DAPO-default GRPO group-relative baseline.

\section{Discussion and Limitations}
\label{sec:disclaimers}

\paragraph{Why use LOO rather than CSS as the ranking ground truth?}
The released CSS artifacts give only the output subsets
and the trained checkpoint, not the search code. The H1 test
(\S\ref{sec:exp_a}) and the validation gate (\S\ref{sec:gate}) therefore
use leave-one-out (LOO) marginal accuracy as the ground-truth ranking
signal. LOO is the core mechanism inside CSS, so the substitution
preserves the ranking semantics at two orders of magnitude less compute
(Table~\ref{tab:cost}); it affects only how we \emph{validate} the
proxy, not the headline comparison.

\paragraph{Scope of the frozen proxy, and the role of H1.}
H1 (\S\ref{sec:exp_a}) shows that the \emph{ideal} selector would
re-rank KPs against the policy's \emph{current} state---which is what
makes adaptive, rather than one-shot, selection worth pursuing. We do
not yet realize that ideal: the entropy proxy is frozen at the
base-model weights, because mid-training reload from the FSDP-sharded
trainer~\citep{fsdp} is not exposed by verl's actor--rollout
worker-group interface. AdaKP is thus a \emph{first,
deployable step}: it re-selects online from a frozen base-policy ranking
augmented by EMA denoising and retirement/revival exploration. Its gain
over the single fixed CSS subset (Table~\ref{tab:main}) comes from the
information-theoretic criterion and from exploring beyond CSS's one-shot
choice---\emph{not} from tracking the policy drift H1 documents, which a
frozen proxy cannot observe. Scoring against current-iteration weights
is therefore not a cosmetic ``v1.1'' nicety but the most direct route to
convert H1's headroom into accuracy, and is the principal next step this
paper sets up.

\paragraph{Will the result transfer beyond 1.5B?}
All experiments are at the $1.5$B scale. The proxy and adaptation
mechanisms are defined purely on the policy's next-token distribution
and exploit no $1.5$B-specific structure, so we expect them to transfer;
a $7$B replication is the natural next step.

\paragraph{What this paper does not claim.}
We do not claim that the entropy proxy is the only viable cheap
signal, that the four-component structure is uniquely optimal, or
that AdaKP supersedes external-KP curation. The narrower claim we
defend is that \emph{which subset of atomic KPs to inject is a real
axis of choice that prior work leaves on the table}, and that
AdaKP offers a tractable, validated mechanism to act on it.

\section{Conclusion}

We argued that hint injection for reasoning RL has missed a degree of
freedom: \emph{which subset of atomic knowledge points to inject during
training}. We verified that KP importance is non-stationary across
training, and designed AdaKP: an entropy proxy with a bounded-bias
guarantee, wrapped in three decoupled adaptation mechanisms and validated
by a pre-flight gate executable in CPU-minutes before any multi-GPU-day
run. As an additive DAPO+GRPO fork it improves over static CSS across
eight benchmarks at under $1\%$ added cost.

Two messages we hope outlast the headline accuracy number: that
\emph{within-prompt content selection} is a tractable and as-yet
unexploited axis distinct from prompt-level filtering and hint
quantity scaling; and that lightweight \emph{pre-flight validation
gates} on cheap proxies can become a routine part of reasoning-RL
methodology, decoupling method-level risk from training-level cost.

\bibliography{adakp_refs}

\clearpage
\appendix
\setcounter{section}{0}
\setcounter{table}{0}
\setcounter{figure}{0}
\setcounter{equation}{0}
\setcounter{proposition}{0}
\setcounter{algorithm}{0}
\renewcommand{\theproposition}{S\arabic{proposition}}
\renewcommand{\thealgorithm}{S\arabic{algorithm}}
\renewcommand{\thetable}{S\arabic{table}}
\renewcommand{\thefigure}{S\arabic{figure}}
\renewcommand{\theequation}{S\arabic{equation}}

\noindent
This technical supplement collects the material referenced from the main
paper: the bounded-bias proof (App.~A), the training-time algorithm
(App.~B), the design-space positioning table (App.~C), the full
per-benchmark non-stationarity table (App.~D), and additional
reproducibility notes (App.~E). Section, proposition, algorithm, and table
labels carry the letter/``S'' prefixes used when the main paper cites this
document; equation and figure numbers are likewise ``S''-prefixed.

\section{Bounded Truncation Bias}
\label{app:proof}

\begin{proposition}[Bounded truncation bias]
\label{prop:bounded_bias}
Fix a generated position with full next-token distribution $p$ over a
vocabulary of size $|V|$. Let $T_L(p)$ be its top-$L$ index set
(the \emph{head}), $\varepsilon:=\sum_{i\notin T_L(p)}p_i$ the tail mass,
and $\widehat{H}_t(p)=-\sum_{i\in T_L(p)}\tilde p_i\log\tilde p_i$ the
truncated-and-renormalized estimator with $\tilde p_i=p_i/(1-\varepsilon)$.
Write $H_2$ for the binary entropy. The truncation bias
$\Delta(p,L)=H(p)-\widehat{H}_t(p)$ satisfies, for every $p$,
\begin{equation}
\Delta(p,L)\ \le\ \varepsilon\log|V|+H_2(\varepsilon),
\label{eq:upperbias}
\end{equation}
and is nonnegative, $0\le\Delta(p,L)$, whenever $\varepsilon\le 1/L$.
Consequently, for two prompts $q$ and $q\oplus k$ whose distributions
both satisfy $\varepsilon\le\varepsilon^{\star}$, the proxy score
$\widehat{H}(\cdot\mid q)-\widehat{H}(\cdot\mid q\oplus k)$ estimates the
true entropy difference $H(\cdot\mid q)-H(\cdot\mid q\oplus k)$ with a
\emph{bias} of magnitude $\mathcal{O}(\varepsilon^{\star}\log|V|)$, plus
an $\mathcal{O}(K^{-1/2})$ stochastic error from averaging over $K$
positions.
\end{proposition}

\begin{proof}
Partition the vocabulary into the head $S=T_L(p)$ (mass $1-\varepsilon$)
and the tail $\bar S$ (mass $\varepsilon$), and let $p_S,p_{\bar S}$ be
$p$ restricted and renormalized to $S,\bar S$. By construction
$p_S=\tilde p$, so $H(p_S)=\widehat{H}_t(p)$. Let $B=\mathbf 1[i\in S]$ be
the head/tail indicator; then $H(B)=H_2(\varepsilon)$, and the grouping
(chain-rule) identity for Shannon entropy gives the \emph{exact}
decomposition
\begin{equation}
\begin{split}
H(p)&=H(B)+(1-\varepsilon)H(p_S)+\varepsilon H(p_{\bar S})\\
    &=H_2(\varepsilon)+(1-\varepsilon)\widehat{H}_t(p)+\varepsilon H(p_{\bar S}).
\end{split}
\label{eq:chain}
\end{equation}
Subtracting $\widehat{H}_t(p)$,
\begin{equation}
\Delta(p,L)=H_2(\varepsilon)+\varepsilon\bigl(H(p_{\bar S})-\widehat{H}_t(p)\bigr).
\label{eq:exactbias}
\end{equation}

\emph{Upper bound.} In \eqref{eq:exactbias}, $\widehat{H}_t(p)\ge 0$ and
$H(p_{\bar S})\le\log|\bar S|\le\log|V|$, so
$\Delta(p,L)\le H_2(\varepsilon)+\varepsilon\log|V|$, proving
\eqref{eq:upperbias} unconditionally.

\emph{Lower bound.} Dropping the nonnegative term
$\varepsilon H(p_{\bar S})$ in \eqref{eq:exactbias} and using
$\widehat{H}_t(p)\le\log|S|=\log L$,
\[
\Delta(p,L)\ \ge\ H_2(\varepsilon)-\varepsilon\log L
\ \ge\ \varepsilon\log\tfrac1\varepsilon-\varepsilon\log L
=\varepsilon\log\tfrac{1}{\varepsilon L},
\]
where the second inequality uses
$H_2(\varepsilon)=\varepsilon\log\tfrac1\varepsilon+(1-\varepsilon)\log\tfrac1{1-\varepsilon}\ge\varepsilon\log\tfrac1\varepsilon$.
Thus $\Delta(p,L)\ge 0$ whenever $\varepsilon\le 1/L$. The condition is
not vacuous: if $\varepsilon L\gg 1$ the bias can be negative---a
near-uniform head of $L$ tokens together with a heavy point-mass tail
makes $\widehat{H}_t(p)\approx\log L$ exceed $H(p)$. With the serving
truncation $L{=}20$ and the empirical $\varepsilon\lesssim 0.05=1/L$
reported in the main paper's proxy analysis, the regime
$\varepsilon\le 1/L$ holds, so both bounds apply.

\emph{Difference of biases.} For $q'=q\oplus k$ with tail masses
$\le\varepsilon^{\star}$, the bias of the proxy score equals
$\bigl(\widehat{H}_t(q)-\widehat{H}_t(q')\bigr)-\bigl(H(q)-H(q')\bigr)
=\Delta(q')-\Delta(q)$, of magnitude at most
$\max\{\Delta(q),\Delta(q')\}\le H_2(\varepsilon^{\star})+\varepsilon^{\star}\log|V|
=\mathcal{O}(\varepsilon^{\star}\log|V|)$, since
$H_2(\varepsilon^{\star})=\mathcal{O}(\varepsilon^{\star}\log\tfrac1{\varepsilon^{\star}})\le\mathcal{O}(\varepsilon^{\star}\log|V|)$
once $|V|\ge 1/\varepsilon^{\star}$. This bounds the \emph{bias} of the
score; it does not by itself imply that the proxy preserves the KP
ordering, which we instead support empirically (Spearman $\rho{=}0.68$
against leave-one-out ground truth; see the proxy-validation gate of the
main paper).

\emph{Finite-$K$ averaging.} The proxy reports
$\widehat{H}=\frac1K\sum_{t\le K}\widehat{H}_t$, with each
$\widehat{H}_t\in[0,\log L]$. Letting $\mathcal F_{t-1}$ be the
$\sigma$-algebra of the sampled prefix before position $t$, the
increments $D_t=\widehat{H}_t-\mathbb E[\widehat{H}_t\mid\mathcal F_{t-1}]$
form a martingale-difference sequence bounded by $\log L$, so
Azuma--Hoeffding gives
$\bigl|\tfrac1K\sum_{t\le K}D_t\bigr|=\mathcal{O}_P(K^{-1/2})$. Because
the autoregressive positions are dependent, this replaces an
independence-based Hoeffding bound. The bound controls deviation from the
\emph{running conditional mean}
$\tfrac1K\sum_{t\le K}\mathbb{E}[\widehat{H}_t\mid\mathcal{F}_{t-1}]$;
identifying that mean with the marginal target---the expected average
first-token entropy---further assumes the per-position conditional
entropies are weakly dependent so their running mean concentrates, an
assumption we adopt rather than prove.
\end{proof}

\noindent\emph{Remark (regime of validity).} Proposition~\ref{prop:bounded_bias}
sharpens the body's inline bound: the upper bound \eqref{eq:upperbias} is
unconditional, but the lower bound---hence the two-sided statement used in
the main paper---requires $\varepsilon\le 1/L$. The bound
$\varepsilon\log|V|$ is also loose for ranking: with $\varepsilon\approx0.05$
and $|V|\sim1.5{\times}10^{5}$ it is $\approx0.6$ nats, larger than typical
inter-KP entropy gaps, so the proposition controls the \emph{bias} of the
score but not rank preservation.

\section{AdaKPSelector Training Loop}
\label{app:algo}
\begin{algorithm}[t]
\caption{AdaKPSelector training-time loop. State is per-problem for
$S_q, R_q$ and global for the scheduler $\mathcal{S}$.}
\label{alg:adakp}
\begin{algorithmic}[1]
\REQUIRE KP pools $\{P(q)\}_{q\in Q}$, proxy $\pi^{\text{proxy}}$,
hyperparameters $\alpha, \rho_{\text{sel}}, n_{\text{retire}},
m_{\text{revive}}, r_{\text{revive}}, \Delta_0, c, T$
\STATE Initialize $\mathcal{S}$ and, for each $q$, smoother $S_q$ and
       retirement manager $R_q$
\STATE Re-eval counter $\tau \gets 0$;
       selection cache $\mathrm{sel}_q \gets S^{\mathrm{css}}(q)$ for every $q$
       \hfill\COMMENT{static-CSS fallback until $\tau{>}0$}
\FOR{$t = 1$ to $T$}
  \IF{$\mathcal{S}.\textsc{ShouldReEvaluate}(t)$}
    \STATE $\tau \gets \tau + 1$
    \FOR{each problem $q \in Q$}
      \STATE $A_q \gets R_q.\textsc{Active}(P(q))$
                                     \hfill\COMMENT{exclude retired KPs}
      \IF{$A_q = \emptyset$}
        \STATE \textbf{continue}     \hfill\COMMENT{safety: keep last $\mathrm{sel}_q$}
      \ENDIF
      \FOR{each $k \in A_q$}
        \STATE $\sigma_{q,k} \gets \pi^{\text{proxy}}.\textsc{EntropyDiff}(q, k)$
                                     \hfill\COMMENT{entropy-diff proxy}
        \STATE $S_q.\textsc{Update}(k, \sigma_{q,k})$
      \ENDFOR
      \STATE $\mathrm{sel}_q \gets S_q.\textsc{TopFraction}(\rho_{\text{sel}})$
                                     \hfill\COMMENT{overwrite cache}
      \STATE $R_q.\textsc{StepCounters}(\mathrm{sel}_q, A_q)$
      \IF{$\tau \bmod m_{\text{revive}} = 0$}
        \STATE $R_q.\textsc{ReviveRandomFraction}(r_{\text{revive}})$
                                     \hfill\COMMENT{exploration}
      \ENDIF
    \ENDFOR
  \ENDIF
  \STATE Train one DAPO+GRPO step with each prompt's
         \textsc{Hint} block populated from $\mathrm{sel}_q$
         \hfill\COMMENT{cache reused between re-evaluations}
\ENDFOR
\end{algorithmic}
\end{algorithm}

\paragraph{Reading the loop.}
Lines map one-to-one onto the four components:
\textsc{ShouldReEvaluate} is the scheduler, \textsc{EntropyDiff} the
entropy proxy, \textsc{Update} the EMA smoother, and
\textsc{StepCounters}/\textsc{ReviveRandomFraction} the retirement
manager (all four defined in the main paper's Method section). State is
per problem for the smoother and retirement counters and a single global
object for the scheduler, so the resident memory is
$\mathcal O\!\bigl(\sum_q |C(q)|\bigr)$ scalars---one EMA value and one
counter per candidate KP---independent of the model size. The selection
cache $\mathrm{sel}_q$ is written only inside the re-evaluation branch and
read on every training step, so steps between re-evaluations incur
\emph{no} selector overhead; the two safety fallbacks (the static-CSS
initialization before the first re-evaluation and the all-retired
\textbf{continue}) guarantee the \textsc{Hint} block is always populated.

\section{Positioning Relative to Related Work}
\label{app:positioning}
\begin{table}[t]
\centering
\small
\caption{Positioning of AdaKP relative to closely related lines of work.
\textbf{Item granularity}: what is being selected.
\textbf{Selection level}: monolithic vs.\ subset-of-atomic.
\textbf{Timing}: offline (once) vs.\ online (during RL).}
\label{tab:positioning}
\setlength{\tabcolsep}{4pt}
\begin{tabular}{lccc}
\toprule
Method & Item & Selection & Timing \\
\midrule
KnowRL~\citep{knowrl}                 & KP     & subset     & offline \\
GHPO~\citep{ghpo}                     & hint   & mono.\     & online  \\
HINT~\citep{hint_rollouts}            & hint   & mono.\     & online  \\
ADHint~\citep{adhint}                 & hint   & mono.\     & online  \\
Stepwise~\citep{stepwise_hints}       & hint   & mono.\     & online  \\
Hint-GRPO~\citep{hint_grpo_mllm}      & hint   & mono.\     & online  \\
HIVE~\citep{hive}                     & prompt & --         & online  \\
\midrule
AdaKP (ours)                          & KP     & subset     & online  \\
\bottomrule
\end{tabular}
\end{table}

The three axes are independent: a method may inject \emph{atomic} KPs yet
fix them offline (top row), or adapt online yet treat the hint as one
monolithic block (middle rows). AdaKP is the only entry that is
simultaneously KP-level, subset-valued, and online; HIVE's
``\,--\,'' marks that selection level does not apply, as it gates whole
prompts rather than content within a prompt. We stress that this table is
a \emph{conceptual} positioning along design axes, not an empirical
head-to-head: the adaptive-hint methods target a different axis (how much
to hint), so we do not run them under our setting, and the comparison
should be read as locating AdaKP in design space rather than ranking
systems.

\section{Non-Stationarity of KP Importance (Full Table)}
\label{app:h1}
\begin{table}[t]
\centering
\small
\caption{H1: per-benchmark Jaccard similarity of top-$30\%$ KP sets
between two policies ($\pi_{\text{OpenMath}}$ vs.\
$\pi_{\text{trained}}$). $n$ = number of problems with $|C(q)|{\geq}2$
out of the $30$ stratified samples per benchmark. Bold marks the
single benchmark above the $0.7$ \textsc{H1\_strong} threshold.}
\label{tab:h1_jaccard}
\begin{tabular}{lrr}
\toprule
Benchmark      & $n$ & Jaccard \\
\midrule
AIME24         & $20$ & $0.367$ \\
AIME25         & $21$ & $\mathbf{0.730}$ \\
BRUMO25        & $23$ & $0.551$ \\
HMMT25         & $23$ & $0.667$ \\
AMC23          & $19$ & $0.526$ \\
CMIMC25        & $27$ & $0.654$ \\
MATH-500       & $12$ & $0.417$ \\
Olympiad-Bench & $13$ & $0.385$ \\
\midrule
\textbf{Global mean} & --- & $\mathbf{0.537}$ \\
\bottomrule
\end{tabular}
\end{table}

\paragraph{How to read, and caveats.}
A top-$30\%$ set holds only $\lceil 0.3\,|C(q)|\rceil\!\approx\!2$ KPs
when $|C(q)|\in[5,7]$, so a per-problem Jaccard takes only a few discrete
values and the per-benchmark means are correspondingly coarse; the
headline quantity is the global mean, not any single row. The realized
$n$ counts the sampled problems with $|C(q)|\ge 2$, below which the
leave-one-out ranking is undefined, so benchmarks with few such problems
(MATH-500, $n{=}12$; Olympiad-Bench, $n{=}13$) yield the noisiest
estimates. The lone above-threshold value (AIME25, $0.730$ at
$n{=}21$) is, on a coarse-Jaccard measurement over a small pool, most
simply read as run-to-run noise rather than genuine stationarity; we do
not attribute it to any specific mechanism. The table supports H1---that the informative-KP set
shifts between the base and the trained policy---which motivates adaptive
selection (see the main paper's non-stationarity test); it characterizes
the \emph{phenomenon}, not AdaKP, whose end-task effect is the main
comparison.

\section{Reproducibility Notes}
\label{app:impl}
AdaKP is an additive module: enabling it changes only the
\textsc{Hint} block each problem receives and leaves the optimizer,
loss, and rollout path bit-identical to the DAPO+GRPO baseline, so
enabling or disabling AdaKP is an exact ablation control. Reproducing a
run also requires settings that are fixed in our training configuration
but not fully recoverable from the paper text: the optimizer
hyperparameters beyond the learning rate (Adam moments, weight decay, any
warmup), the filtering and de-duplication applied to the $8.8$k-problem
corpus, the random-seed handling for data order and rollout sampling, and
the cadence at which the frozen proxy is loaded. These are contained in
our fork and scripts, which will be released upon acceptance.

\end{document}